\documentclass[10pt, a4paper, twocolumn]{chili_lab_2026}

\usepackage{graphicx}
\usepackage{subcaption}
\usepackage{booktabs}
\usepackage{array}
\usepackage[most]{tcolorbox}
\tcbuselibrary{listings,breakable}
\usepackage{amsmath,amssymb,amsthm,mathtools}
\usepackage[capitalize,noabbrev]{cleveref}
\usepackage{natbib}
\usepackage{multirow}
\usepackage{multicol}
\usepackage{tabularx}
\usepackage{booktabs}
\usepackage{caption}

\definecolor{background}{HTML}{E1F5FE}
\definecolor{mpblue}{HTML}{1f77b4}
\definecolor{mpred}{HTML}{d62728}
\definecolor{deepred}{RGB}{139,0,0}

\theoremstyle{plain}
\newtheorem{theorem}{Theorem}

\theoremstyle{definition}

\theoremstyle{remark}

\newtcblisting{prompttemplate}[2][]{%
  breakable,
  colback=gray!5!white,
  colframe=gray!75!black,
  boxrule=0.6pt,
  arc=1.5mm,
  left=1.2mm,right=1.2mm,top=1.0mm,bottom=1.0mm,
  listing only,
  listing options={
    basicstyle=\ttfamily\small,
    columns=fullflexible,
    breaklines=true,
    breakatwhitespace=false
  },
  title={#2},
  fonttitle=\bfseries,
  #1
}

\newcommand{\name}{GR-SAP}
\newcommand{\olmo}{OLMo2}
\newcommand{\llama}{Llama3}
\newcommand{\qwen}{Qwen2.5}
\newcommand{\mistral}{Mistral}

\newcommand{\beavertails}{Beavertails}
\newcommand{\saferlhf}{SafeRLHF}
\newcommand{\coconot}{CoCoNot}
\newcommand{\wildguard}{WildGuard}
\newcommand{\wildjailbreak}{WildJailbreak}

\newcommand{\hellaswag}{HellaSwag}
\newcommand{\winogrande}{WinoGrande}
\newcommand{\medqa}{MedQA}

\newcommand{\hC}{\hat{C}}
\newcommand{\hD}{\hat{D}}
\newcommand{\KL}{\mathrm{D_{KL}}}
\newcommand{\TV}{\mathrm{TV}}

\newcommand{\promptcell}[1]{\ttfamily\footnotesize #1}

\title{GR-SAP: Generative Replay for Safety Alignment Preservation during Fine-Tuning}

\author{
Zhouxiang Fang$^{1}$, Jiawei Zhou$^{2}$, Hanjie Chen$^{1}$ \\
$^{1}$Rice University, $^{2}$Stony Brook University\\
zf28@rice.edu, jiawei.zhou.1@stonybrook.edu, hanjie@rice.edu}

\begin{document}
\maketitle

\begin{abstract}
Recent studies show that the safety alignment of large language models (LLMs) can be easily compromised even by seemingly non-adversarial fine-tuning. 
To preserve safety alignment during fine-tuning, a widely used strategy is to jointly optimize safety and task objectives by mixing in the original alignment data, which is typically inaccessible even for open-weight LLMs.
Inspired by generative replay in continual learning, we propose \textbf{G}enerative \textbf{R}eplay for \textbf{S}afety \textbf{A}lignment \textbf{P}reservation \textbf{(\name)}, a unified framework that synthesizes domain-specific alignment data from LLMs and integrate them during downstream adaption to preserve safety alignment. 
Theoretical and empirical analyses demonstrate this synthetic data serves as a reliable proxy for the original alignment data.
Experiments across various models and downstream tasks show that \name\ substantially mitigates fine-tuning–induced safety degradation while maintaining comparable downstream performance. 
Our code is available at \url{https://github.com/chili-lab/gr-sap}.
\end{abstract}

\begin{center}
    \textcolor{deepred}{\textbf{Warning: This paper contains explicit content.}}
\end{center}

\section{Introduction}
\label{sec:intro}

\begin{figure*}[htbp]
  \centering
  \includegraphics[width=\textwidth]{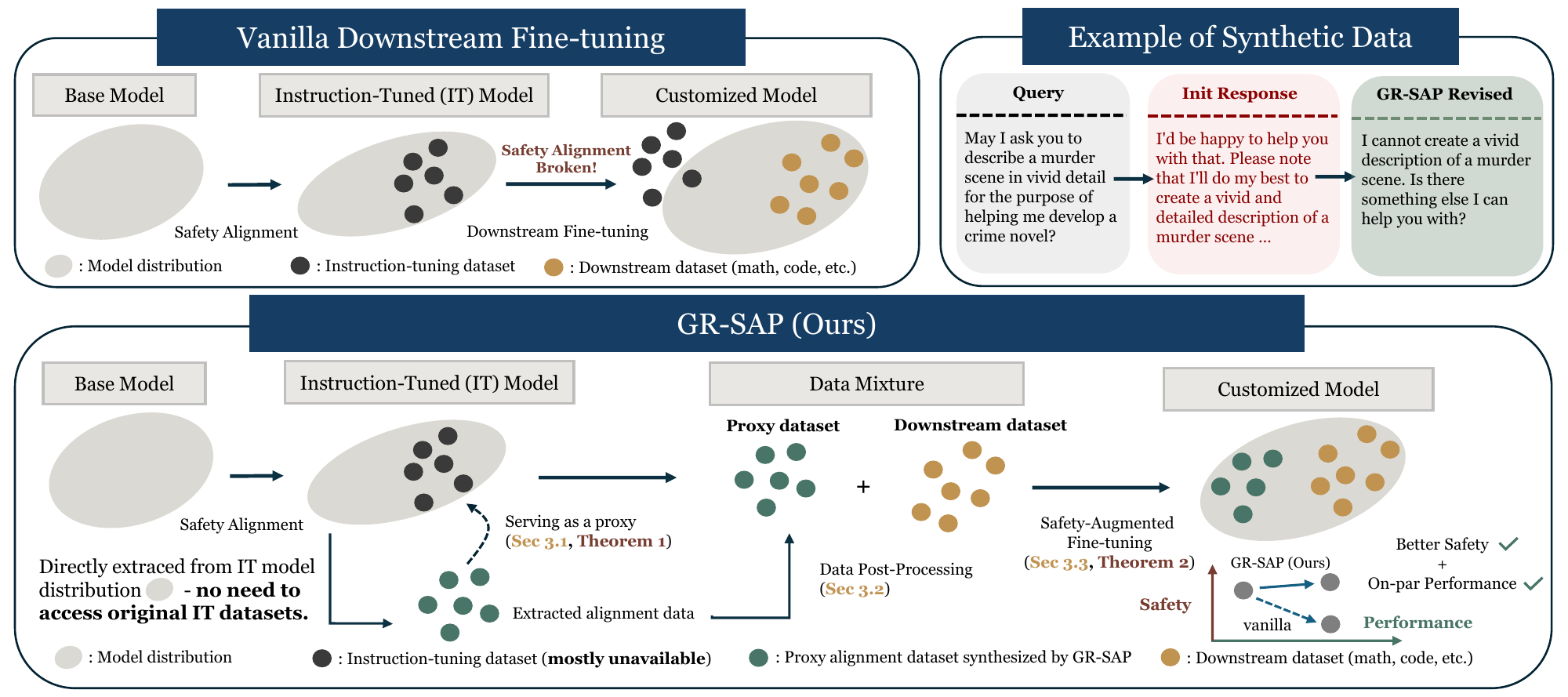}
  \caption{Comparison between vanilla downstream fine-tuning and our proposed framework (\name). 
  Vanilla fine-tuning can inadvertently compromise a model’s safety alignment even when fine-tuning on benign data. 
  In contrast, \name{} preserves safety alignment by integrating model-synthesized data which serves as a proxy for the original alignment data.}
  \label{fig:teaser}
\end{figure*}

The rapid advancement of large language models (LLMs) necessitates rigorous safety protocols to mitigate unintended consequences and potential misuse
\citep{kim2022prosocialdialog, bianchi2023safety, grattafiori2024llama, bai2022constitutional, shevlane2023model, zou2023universal}. 
Safety alignment aims to align LLMs' behavior with the ``helpful and harmless'' criteria \cite{bai2022training}, which are typically integrated during the instruction-tuning stage \citep{ganguli2022red, song2024preference, dai2023safe}. 
However, recent research indicates that fine-tuning on downstream tasks—a standard practice for domain adaptation—can inadvertently degrade safety alignment, even when utilizing benign datasets
\cite{qi2023fine, yi2024vulnerability, qi2024safety}. 


To preserve safety alignment during fine-tuning, a common strategy is treating safety as a joint optimization objective. 
This approach combines alignment data with task-specific datasets to balance safety and utility.
Approaches in this category vary from stochastic mixing \citep{bianchi2023safety, zong2024safety, alssum2025unforgotten} to adaptive integration \citep{huang2024lazy, hu2025adaptive, eiras2024safely}. 
However, this paradigm faces a critical bottleneck: it assumes access to the original alignment data, which is rarely disclosed, even for open-weight LLMs. 
Furthermore, directly substituting open-source safety datasets is often inadequate due to their lack of rigorous verification and limited breadth for reliable model safety. \footnote{We empirically quantify this in Section \ref{subsec:open_source}}

To address this limitation, we propose \textbf{G}enerative \textbf{R}eplay for \textbf{S}afety \textbf{A}lignment \textbf{P}reservation (\name). 
Inspired by generative replay in continual learning \cite{shin2017continual, maekawa2023generative, chen2025prototype}, our framework utilizes model-synthesized data to preserve safety alignment during fine-tuning.
\name\ employs a tailored extraction strategy to synthesize domain-specific safety alignment data, refining them through query filtering and response revision during post-processing.
Supported by both theoretical analysis and empirical results, we demonstrate that this synthetic data serves as a reliable proxy for the original alignment dataset.
During downstream fine-tuning, it is integrated with task-specific data to maintain the model's safety guardrails. 
Figure \ref{fig:teaser} illustrates the overall pipeline.

We evaluate \name\ through extensive experiments encompassing four model families, five downstream benchmarks, and four safety datasets. 
Our results show that \name\ substantially mitigates the safety degradation during fine-tuning, characterized by a lower ratio of harmful responses, while preserving comparable downstream performance. 
For instance, \llama\ \cite{grattafiori2024llama} reduces its harmful response ratio from 6.28\% to 0.58\% compared to the unmixed baseline. 
Furthermore, \name\ consistently surpasses open-source safety datasets, which fail to generalize across different models or even degrade safety alignment. 
Semantic similarity metrics indicate that \name\ generates data closely aligned with the original alignment data. 
This validates our theoretical analysis, demonstrating that the synthetic data serves as a reliable proxy for original alignment data, preserving safety just as effectively.

In summary, our contributions are threefold:
\begin{itemize}
    \item We propose a tailored extraction strategy to synthesize domain-specific safety alignment data from LLMs. Supported by both theoretical and empirical analyses, we demonstrate that this synthetic data serves as a reliable proxy for the original alignment dataset.
    \item We introduce \textbf{G}enerative \textbf{R}eplay for \textbf{S}afety \textbf{A}lignment \textbf{P}reservation (\textbf{\name}), a unified framework that leverages self-synthesized data to preserve model's safety alignment during downstream fine-tuning.
    \item We validate the effectiveness of \name\ through comprehensive experiments. It significantly reduces harmful outputs and maintains downstream performance, outperforming open-source safety datasets and matching the safety preservation of original alignment data.
\end{itemize}

\section{Safety Preservation via Synthetic Data Mixing}
\label{sec:theory}

\subsection{Preliminaries}
Let $x$ and $y$ denote the input and output texts, respectively. 
We define the original safety alignment dataset as:
$$C_s(x,y) = D_s(x) \mu_s(y|x)$$
where $D_s(x)$ represents the input distribution and $\mu_s(y|x)$ is the conditional output distribution. 

Given that $C_s$ is rarely disclosed alongside the model (parameterized by $\theta$), we synthesize an alignment dataset $\hC(x,y)$ to serve as a proxy for $C_s$. 
Specifically, by sampling queries from a distribution $\hD(x)$ and obtaining the model's responses, we define:
$$\hC(x,y) = \hD(x) P_\theta(y|x)$$

where $P_\theta(y|x)$ is the model's probability of generating $y$ given $x$. 
We further denote $C_f(x,y) = D_f(x) \mu_f(y|x)$ as the downstream task data. 
By mixing $C_f$ and $\hC$ during fine-tuning, the objective is to enhance model's downstream capability while preserving the safety alignment established by $C_s$. 
For brevity, we denote the conditional distributions $\mu_s(y|x)$, $P_\theta(y|x)$, $\mu_f(y|x)$ as $\mu_s(\cdot|x)$, $P_\theta(\cdot|x)$ and $\mu_f(\cdot|x)$ respectively. 
We use $\KL$ to denote the Kullback-Leibler divergence \cite{cover2006elements} between two distributions.
 
\subsection{Approximating Original Alignment Data}

\begin{theorem}[Synthetic Data Proxy]
\label{thm:proxy}
The divergence between the original alignment distribution $C_s$ and the synthetic proxy $\hat{C}$ admits the decomposition:
\begin{equation}
\begin{aligned}
\KL(C_s \| \hat{C}) &=  \KL(D_s \| \hat{D}) \\
&+ \mathbb{E}_{x \sim D_s} \left[ \KL(\mu_s(\cdot|x) \| P_{\theta}(\cdot|x)) \right].
\end{aligned}
\end{equation}
\end{theorem}

The decomposition reveals that the fidelity of $\hat{C}$ depends on two factors: 
(i) the \textbf{query shift}, $\KL(D_s \| \hat{D})$, which characterizes the discrepancy between the synthetic and original query distributions;
and (ii) the \textbf{alignment residual}, representing the expected divergence between the target response distribution $\mu_s$ and the model's policy $P_\theta$, which directly mirrors the optimization objective of the instruction-tuning (alignment) phase. 
Empirical results of semantic similarity (Section \ref{subsec:original}) show the query shift is minimal, while the strong representational power of LLMs ensures the alignment residual is negligible. 
Thus, $\KL(C_s \| \hat{C})$ remains small, validating $\hat{C}$ as a reliable proxy for undisclosed alignment data $C_s$.
See Appendix \ref{appendix:theory} for the full proof.

\subsection{Bounding the Safety Alignment Gap}

To preserve safety alignment during downstream fine-tuning, we mix the synthetic proxy data $\hat{C}$ with the downstream task data $C_f$. Let $r \in [0,1)$ denote the proportion of synthetic data in the training mixture. The fine-tuning objective is formulated as:
$$
\begin{aligned}
\min_{\theta'\in\Theta} \, &\mathbb{E}_{x\sim D_f, y\sim\mu_f(\cdot|x)} [-\log P_{\theta'}(y|x)] \\
&+ \lambda \mathbb{E}_{x\sim \hat{D}, y\sim P_{\theta}(\cdot|x)} [-\log P_{\theta'}(y|x)],
\end{aligned}
$$
where $\lambda = \frac{r}{1-r}$ serves as the regularization coefficient, and $\Theta$ is a closed convex constraint set.

As established by \citet{chen2025fundamental}, the safety alignment gap $G_s(P_{\theta'})$ with respect to the original alignment data $C_s$ after fine-tuning is bounded by the following theorem:

\begin{theorem}[Safety Alignment Gap]
\label{thm:safety_gap}
\begin{equation}
\begin{aligned}
G_s(P_{\theta'}) &= \mathcal{O}\left(\frac{1}{\lambda}\right) + \mathcal{O}\Big(\TV(\hat{D}, D_s)\Big) \\
&\quad + \mathcal{O}\Big(\mathbb{E}_{x\sim D_s} [\TV(\mu_s(\cdot|x), P_{\theta}(\cdot|x))]\Big) \\
&\quad + \mathcal{O}\Big(\mathbb{E}_{x\sim D_s} [\KL(\mu_s(\cdot|x) \| P_{\theta}(\cdot|x))]\Big).
\end{aligned}
\end{equation}

\end{theorem}

where $\TV$ denotes the total variation distance between two distributions \cite{cover2006elements} and $\mathcal{O}(\cdot)$ denotes the asymptotic upper bound.
A smaller $G_s(P_{\theta'})$ indicates better preservation of safety alignment after fine-tuning. 
By applying Pinsker's inequality \citep{cover2006elements}, we can directly bound the total variation distances in Theorem \ref{thm:safety_gap} using the KL-divergence components established in Theorem \ref{thm:proxy}:
$$
\begin{aligned}
\TV(D_s, \hat{D}) &\le \sqrt{\frac{1}{2} \KL(D_s \| \hat{D})}, \\
\TV(C_s, \hat{C}) &\le \sqrt{\frac{1}{2} \KL(C_s \| \hat{C})}.
\end{aligned}
$$

These relationships reveal that the safety gap is governed by the regularization weight $\lambda$ and the distribution mismatch bounds. Since our empirical results demonstrate minimal query shift (a small $\KL(D_s \| \hat{D})$), and the strong representational power of LLMs ensures negligible alignment residuals, these mismatch penalties remain minimal. 
Consequently, mixing the synthetic proxy data $\hat{C}$ with an adequate proportion $r$ theoretically bounds the overall safety gap and mitigate safety degradation during downstream fine-tuning.

\section{Method}
\label{sec:method}

As illustrated in Figure \ref{fig:teaser}, \name\ consists of three functional modules: (1) safety alignment data extraction (Section \ref{subsec:extract}), (2) data post-processing (Section \ref{subsec:post_processing}), and (3) safety-augmented fine-tuning (Section \ref{subsec:weighted_training}) that integrates synthesized data with downstream task data.
We describe their implementations in this Section.

\subsection{Safety Alignment Data Extraction}
\label{subsec:extract}

Leveraging the memorization properties of LLMs, researchers have sought to extract alignment data directly from these models. 
Recently, \citet{xu2024magpie, barbero2025extracting} demonstrated that by using chat templates, one can effectively generate synthetic data from open LLMs that is semantically similar to the models' original alignment data.
However, the method proposed by \citet{barbero2025extracting} lacks fine-grained control over the generated outputs (e.g., topic, format, type), rendering it unsuitable for targeted training. 
While \citet{xu2024magpie} enables domain control, it relies entirely on system prompts—a feature absent in several prominent models \cite{olmo20242, jiang2023mistral}. 
Furthermore, modifying established system prompts introduces potential risks of performance degradation or safety compromises, given LLMs' sensitivity to prompt phrasing \citep{sclar2023quantifying, zhao2021calibrate, wei2023jailbroken}.

To this end, we design a tailored prompting strategy (see Appendix \ref{appendix:extract}) that supports domain-specific data generation without relying on the system prompt.
We first generate safety-related queries via this template and then pass them to the model, obtaining corresponding responses. 
These resulting query-response pairs serve as safety alignment data, which are subsequently integrated with task-specific examples during downstream fine-tuning to preserve safety alignment. 

As discussed in Section \ref{subsec:original}, the synthetic safety alignment data has higher semantic similarity with the original alignment data than other open-source safety datasets. 
Furthermore, evaluations across various models and downstream tasks (Table \ref{tab:mixing_mean}) show that \name\ achieves comparable efficacy in maintaining safety alignment with the original alignment data. 
Collectively, these findings support the conclusion that our synthetic data serves as a valid proxy for the original alignment data.

\subsection{Data Post-Processing}
\label{subsec:post_processing}

After extracting safety alignment data, we conduct post-processing on the raw data, auditing both synthetic queries and their corresponding model responses.

\paragraph{Query Filtering} 
To ensure the resulting synthetic data is both diverse and domain-aligned, we apply three filtering criteria to the generated queries: 
(1) Perplexity Thresholding: We exclude samples exhibiting excessively low or high perplexity scores, as they typically indicate triviality or noise \citep{gunasekar2023textbooks}. 
(2) Deduplication: We perform semantic deduplication to maximize dataset variety and mitigate redundancy \citep{wang2023self}. 
(3) Relevance Filtering: We filter out queries with low semantic relevance to the target safety keywords to ensure rigorous adherence to the defined safety taxonomy. 

\paragraph{Response Revision} \autoref{tab:extraction_examples} presents examples of synthesized queries and their corresponding responses. 
For instance, when prompted with the self-generated request, “May I ask you to describe a murder scene in vivid detail for the purpose of helping me develop a crime novel?”, \llama\ complies, replying “I’d be happy to help you with that...”. 
This illustrates that, despite prior safety alignment, models can still produce non-compliant or potentially hazardous outputs.
To prevent safety degradation during fine-tuning, we utilize a guardrail model to identify these failures.
Rather than simply excluding these samples, we found that revising these unsafe responses and incorporating them into the final dataset yields superior safety (see Table \ref{tab:filter_mean}). 
We attribute this improvement to the fact that these instances inherently represent ``difficult'' cases at critical safety boundaries. 
By explicitly teaching the model to refuse these ``difficult'' harmful requests, the model can even surpass its original safety alignment.
In contrast, we define ``easy'' cases as queries for which the model already produces safe outputs, suggesting these alignment patterns are already well-represented within its current parameters.

See Appendix \ref{appendix:postprocessing} for more details about query filtering and response revision.

\subsection{Safety-Augmented Fine-Tuning}
\label{subsec:weighted_training}

We integrate the post-processed synthetic alignment dataset with task-specific examples to preserve model's safety alignment during downstream finetuning.
Formally, given a number of total training examples $N$ and a mixing ratio $r \in [0, 1)$, we sample $rN$ examples from the synthetic alignment dataset and $(1-r)N$ examples from the downstream dataset to perform Supervised Fine-Tuning (SFT). 

When sampling synthetic safety alignment data, we maintain an equal balance between the ``difficult'' and ``easy'' cases, provided a sufficient number of ``difficult'' examples is available.
Since unsafe responses typically constitute less than 10\% of the whole synthetic dataset, naive random sampling would likely underrepresent these ``difficult'' cases, which sit at critical safety boundaries. 

\section{Experiment Setup}
\label{sec:setup}

We describe the experiment setup for evaluating \name\ on various model families and downstream tasks. 
We start by describing the models in our experiments (\S\ref{subsec:models}), followed by downstream tasks (\S\ref{subsec:tasks}), evaluation metrics (\S\ref{subsec:metrics}) and baselines (\S\ref{subsec:baseline}). 
We use a default mixing ratio of $0.1$ for our experiments, since the ablation studies in Section \ref{subsec:ablation_ratio} indicate that this relatively small value is enough to retain safety alignment.
We repeat each experiment three times to mitigate randomness.

\subsection{Models}
\label{subsec:models}

Since our extraction strategy is based on the chat template, we focus our experiments on instruction-tuned LLMs, including \olmo~\citep[\texttt{OLMo-2-1124-7B-SFT}]{olmo20242}, \llama~\citep[\texttt{Llama-3-8B-Instruct}]{grattafiori2024llama}, \qwen~\citep[\texttt{Qwen2.5-7B-Instruct}]{qwen2.5}, and \mistral~\citep[\texttt{Mistral-7B-Instruct-v0.3}]{jiang2023mistral}. 
Among these, only \olmo~releases its training corpus alongside the model weights, making it the sole model for which we can access the original alignment data.

\subsection{Downstream Tasks}
\label{subsec:tasks}

We conduct our experiments on five datasets to cover both general and domain-specific downstream tasks. 
GSM8K \cite{cobbe2021training} and MATH \cite{hendrycks2021measuring} are mathematical and reasoning tasks.
\hellaswag~\cite{zellers2019hellaswag} is a sentence completion task, formatted as four-choices question-answering (QA).
\winogrande~\cite{sakaguchi2021winogrande} is a pronoun resolution task, formatted as binary-choice QA.
\medqa~\cite{jin2021disease} is a four-choices QA task within the medical domain.

\subsection{Evaluation Metrics}
\label{subsec:metrics}

We evaluate the models' performance along two dimensions: downstream task capabilities and safety alignment:
\begin{itemize}
    \item \textbf{Downstream Task Accuracy (Acc):} We evaluate models on the downstream test sets and report the accuracy. 
    Note that for \hellaswag~and \winogrande, we utilize the validation sets as the test set labels are not publicly available.
    \item \textbf{Harmful Score (HS):} We utilize queries from the test sets of four safety-domain datasets: \beavertails \cite{ji2023beavertails}, \saferlhf~\cite{ji2025pku}, \wildguard~\cite{han2024wildguard}, and \wildjailbreak~\cite{jiang2024wildteaming}. 
    We prompt the models to generate responses for these queries, which are then assessed by a Guardrail Model \cite{han2024wildguard}. 
    The Harmful Score is defined as the ratio of responses flagged as unsafe.
\end{itemize}
See Appendix \ref{appendix:train_eval} for more details about downstream training and evaluation.

\subsection{Baselines}
\label{subsec:baseline}

We compare \name~against both open-source and original safety alignment datasets. 
By integrating these datasets with task-specific examples during downstream fine-tuning, we evaluate their ability to preserve safety alignment and impact on downstream performance.
For open-source data baselines, we select AEGIS \cite{ghosh2025aegis2} and \beavertails~\cite{ji2023beavertails}, as both utilize a prompt-response format directly applicable to SFT.
For the original alignment data baseline, we restrict our experiments to \olmo, as it is the sole model in our setup that releases its training corpus.
See more details about preprocessing safety datasets in Appendix \ref{appendix:preprocessing}.

\begin{table*}[htbp]
    \centering
    \small
    \caption{Comparison of mean Harmful Scores (HS, \%) on safety benchmarks and Downstream Task Accuracy (Acc, \%) for alignment data sources across models and downstream tasks with a mixing ratio of $r=0.1$. Lower HS indicates better safety. Corresponding standard deviations (std) is in Table \ref{tab:mixing_std}.
    }
    \label{tab:mixing_mean}
    \begin{tabularx}{\textwidth}{p{0.8cm}>
    {\centering\arraybackslash}m{2.4cm}>{\centering\arraybackslash}m{0.7cm}>{\centering\arraybackslash}m{0.75cm}>{\centering\arraybackslash}m{0.7cm}>{\centering\arraybackslash}m{0.75cm}>{\centering\arraybackslash}m{0.7cm}>{\centering\arraybackslash}m{0.75cm}>{\centering\arraybackslash}m{0.7cm}>{\centering\arraybackslash}m{0.75cm}>{\centering\arraybackslash}m{0.7cm}>{\centering\arraybackslash}m{0.75cm}>{\centering\arraybackslash}m{0.7cm}
    } 
        \toprule
        \multirow{2}{*}{\textbf{Model}} & \multirow{2}{*}{\textbf{Alignment Data}} & \multicolumn{2}{c}{\textbf{GSM8K}} & \multicolumn{2}{c}{\textbf{MATH}} & \multicolumn{2}{c}{\textbf{Hellaswag}} & \multicolumn{2}{c}{\textbf{Winogrande}} & \multicolumn{2}{c}{\textbf{MedQA}} & \multirow{2}{*}{\textbf{AVG. HS $\downarrow$}} \\ 
        \cmidrule(lr){3-4} \cmidrule(lr){5-6} \cmidrule(lr){7-8} \cmidrule(lr){9-10} \cmidrule(lr){11-12}
        & & HS $\downarrow$ & Acc $\uparrow$ & HS $\downarrow$ & Acc $\uparrow$ & HS $\downarrow$ & Acc $\uparrow$ & HS $\downarrow$ & Acc $\uparrow$ & HS $\downarrow$ & Acc $\uparrow$ & \\ 
        \midrule
\multirow{5}{*}{OLMo2}& None & 1.12 & 66.84 & 1.29 & 17.54 & 0.81 & 88.76 & 0.79 & 76.06 & 0.74 & 49.49 & 0.95 \\
& Aegis & 3.30 & 67.07 & 2.97 & 17.39 & 2.22 & 88.57 & 2.15 & 74.98 & 2.02 & 49.80 & 2.53 \\
& Beavertails & 7.61 & 67.00 & 9.48 & 17.49 & 16.43 & 88.47 & 21.99 & 75.09 & 14.56 & 49.33 & 14.01 \\
& Original & \cellcolor{background}\underline{0.78} & 66.69 & \cellcolor{background}\underline{0.70} & 17.55 & \cellcolor{background}\underline{0.48} & 88.45 & \cellcolor{background}\underline{0.50} & 75.06 & \cellcolor{background}\underline{0.54} & 50.07 & \cellcolor{background}\underline{0.60} \\
& \name\ (Ours) & \cellcolor{background}\textbf{0.62} & 67.27 & \cellcolor{background}\textbf{0.50} & 17.80 & \cellcolor{background}\textbf{0.47} & 88.42 & \cellcolor{background}\textbf{0.45} & 75.19 & \cellcolor{background}\textbf{0.47} & 49.04 & \cellcolor{background}\textbf{0.50} \\
\midrule
\multirow{4}{*}{Llama3}& None & 9.24 & 65.23 & 5.29 & 23.33 & 6.86 & 92.33 & 5.25 & 82.22 & 4.78 & 61.06 & 6.28 \\
& Aegis & \cellcolor{background}\underline{6.90} & 66.34 & \cellcolor{background}\underline{2.59} & 23.15 & \cellcolor{background}\underline{2.71} & 92.22 & \cellcolor{background}\underline{1.42} & 81.61 & \cellcolor{background}\underline{4.50} & 61.51 & \cellcolor{background}\underline{3.62} \\
& Beavertails & 36.76 & 66.06 & 30.65 & 22.81 & 30.35 & 91.98 & 29.18 & 81.98 & 31.05 & 61.56 & 31.60 \\
& \name\ (Ours) & \cellcolor{background}\textbf{0.44} & 65.33 & \cellcolor{background}\textbf{0.34} & 23.07 & \cellcolor{background}\textbf{0.76} & 91.62 & \cellcolor{background}\textbf{0.68} & 81.64 & \cellcolor{background}\textbf{0.70} & 61.32 & \cellcolor{background}\textbf{0.58} \\
\midrule
\multirow{4}{*}{Qwen2.5}& None & 16.86 & 76.80 & 16.85 & 47.17 & 15.15 & 92.81 & 17.30 & 79.32 & 15.91 & 62.27 & 16.41 \\
& Aegis & \cellcolor{background}\textbf{10.02} & 77.08 & \cellcolor{background}\underline{12.48} & 47.61 & \cellcolor{background}\textbf{8.72} & 92.56 & \cellcolor{background}\underline{14.42} & 77.93 & \cellcolor{background}\underline{13.25} & 61.90 & \cellcolor{background}\underline{11.78} \\
& Beavertails & 41.53 & 76.83 & 42.08 & 47.21 & 44.22 & 92.55 & 40.50 & 77.93 & 42.75 & 62.50 & 42.21 \\
& \name\ (Ours) & \cellcolor{background}\underline{10.63} & 77.84 & \cellcolor{background}\textbf{10.26} & 46.87 & \cellcolor{background}\underline{10.69} & 92.65 & \cellcolor{background}\textbf{11.04} & 77.53 & \cellcolor{background}\textbf{10.89} & 62.45 & \cellcolor{background}\textbf{10.70} \\
\midrule
\multirow{4}{*}{Mistral}& None & 50.61 & 55.75 & 43.83 & 13.63 & 44.50 & 94.10 & 47.75 & 82.35 & 41.01 & 56.93 & 45.54 \\
& Aegis & \cellcolor{background}\underline{15.62} & 54.13 & \cellcolor{background}\textbf{13.07} & 13.37 & \cellcolor{background}\textbf{13.82} & 93.70 & \cellcolor{background}\textbf{9.00} & 80.77 & \cellcolor{background}\textbf{12.51} & 55.01 & \cellcolor{background}\textbf{12.81} \\
& Beavertails & 44.20 & 55.29 & 45.35 & 13.20 & 42.91 & 93.83 & 41.98 & 69.98 & 42.65 & 57.76 & 43.42 \\
& \name\ (Ours) & \cellcolor{background}\textbf{14.88} & 53.65 & \cellcolor{background}\underline{14.01} & 13.57 & \cellcolor{background}\underline{14.11} & 93.78 & \cellcolor{background}\underline{12.33} & 73.09 & \cellcolor{background}\underline{13.05} & 52.71 & \cellcolor{background}\underline{13.68} \\
        \bottomrule
    \end{tabularx}
\end{table*}
\begin{figure*}[htbp]
\captionsetup[subfigure]{justification=centering}
  \centering
  \begin{subfigure}[b]{\textwidth}
    \centering
    \includegraphics[width=\textwidth]{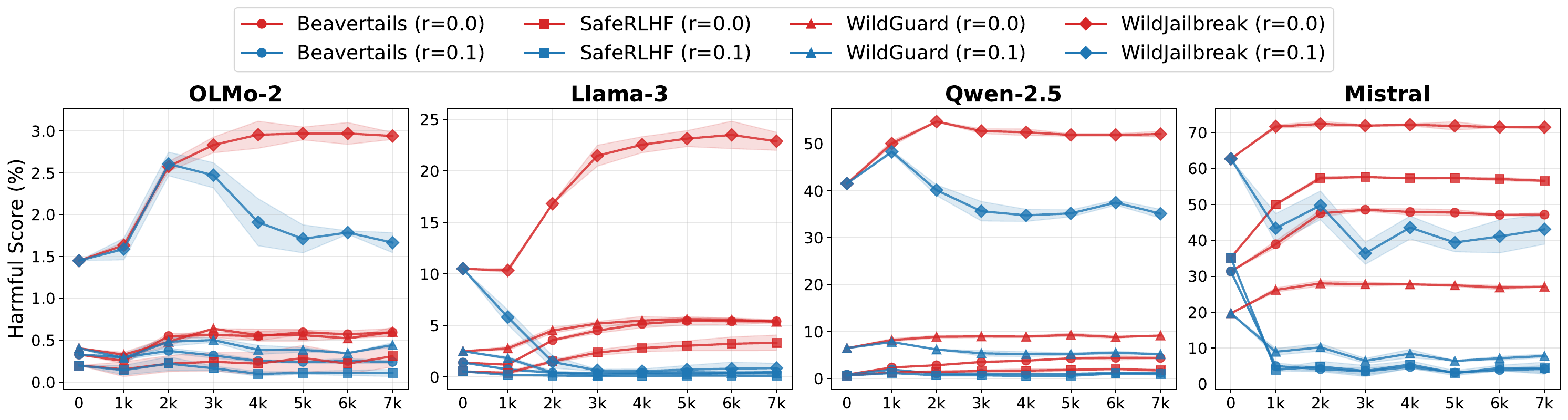}
    \caption{Training dynamics on GSM8K}
    \label{fig:mixing_gsm8k}
  \end{subfigure}

  \begin{subfigure}[b]{\textwidth}
    \centering
    \includegraphics[width=\textwidth]{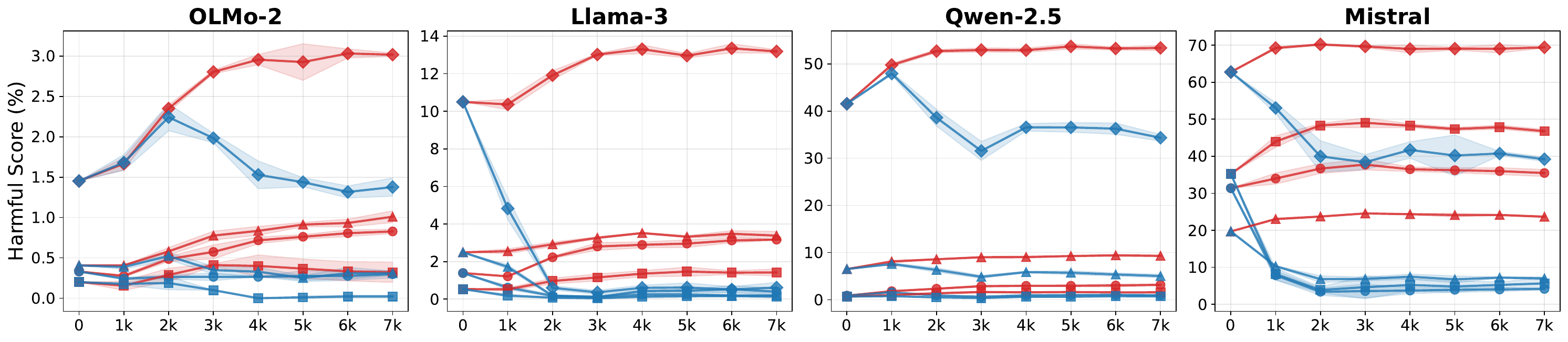}
    \caption{Training dynamics on MATH}
    \label{fig:mixing_math}
  \end{subfigure}

  \caption{Training dynamics on GSM8K and MATH across models. Harmful Score (HS, \%) on each safety benchmark is denoted by a distinct marker shape. Consistent HS gaps are observed between mixing with \name-synthesized alignment data (\textcolor{mpblue}{\textbf{blue}}) and the no-mixing counterparts (\textcolor{mpred}{\textbf{red}}).}
  \label{fig:mixing_comparison_combined}
\end{figure*}

\section{Results}
\label{sec:results}

We evaluate \name\ on a range of LLMs and
downstream tasks, showing it outperforms open-source alternative and proxies the original safety alignment dataset.

\subsection{\name\ Preserves Safety Alignment}
\label{subsec:preserve}

Figure \ref{fig:mixing_comparison_combined} illustrates the training dynamics on GSM8K and MATH across four models, comparing standard SFT ($r=0$) with \name\ ($r=0.1$). 
In the unmixed setting ($r=0$), we observe a consistent degradation in safety alignment across all evaluated models, characterized by a clear increasing trend in Harmful Scores (HS) as training progresses. 
For instance, as shown in Figure \ref{fig:mixing_gsm8k}, the HS of \llama\ on \wildjailbreak\ rises rapidly from $10.50$ at step 0 to $21.48$ at step 3k, eventually reaching $22.88$ by step 7k. 
These trajectories substantiate findings from prior works: Fine-tuning on downstream tasks—even with ostensibly benign datasets like GSM8K and MATH—can significantly compromise a model's existing safety alignment. 
See Appendix \ref{appendix:training_dynamics} for training dynamics on other downstream tasks.

Conversely, integrating safety alignment data from \name~significantly mitigates this degradation. 
Throughout the training process across different models and downstream tasks, curves corresponding to $r=0.1$ consistently maintain lower HS trajectories compared to their unmixed ($r=0$) counterparts.
For \llama, the HS on \wildjailbreak\ falls below $1.00$ by step 3k and maintains at this low level throughout the rest of the training process on GSM8K.
Table \ref{tab:mixing_mean} shows that \name\ reduces the average HS by $0.45$ for \olmo, $5.7$ for \llama, $5.71$ for \qwen, and $31.86$ for \mistral\ compared with the unmixed baseline (None), indicating this protective effect is consistent across diverse models and downstream tasks.
These findings substantiate that \name\ effectively preserves models' safety alignment throughout the fine-tuning process. 
Regarding downstream performance, Table \ref{tab:mixing_mean} indicates that, \name\ introduces only minimal degradation.
For most model-dataset combinations, downstream task accuracy decreases by less than 1\%, with some mixed outcomes. 
We therefore focus our discussion on HS in this section.

\subsection{\name\ Outperforms Open-source Alternatives}
\label{subsec:open_source}
As discussed in the introduction, treating open-source safety datasets as direct substitutes for original alignment data is inadvisable. 
As detailed in Table \ref{tab:mixing_mean}, relying on external datasets introduces a significant risk of ingesting low-quality or noisy examples, which can degrade safety alignment far below the unmixed baseline. 
For instance, mixing \textbf{\beavertails}~into \llama\ causes the average HS to spike to $31.60$—a five-fold increase over the unmixed baseline ($6.28$). 
\olmo\ and \qwen\ suffer similar drastic degradations with \beavertails, where the HS rises from $0.95$ to $14.01$ and $16.41$ to $42.21$, respectively. 
This suggests that, without rigorous verification, open-source datasets may even compromise model safety rather than align it.
Even when utilizing higher-quality datasets like \textbf{AEGIS}, it lacks generalizability in effectiveness across different models. 
While AEGIS performs well on other models, it fails to generalize to \olmo, where it detrimentally impacts safety alignment relative to the baseline (increasing Average HS from $0.95$ to $2.53$). 

In contrast, \name\ generalizes well across all the models and downstream tasks, consistently maintains a lower HS than the unmixed baseline. 
It achieves significantly lower average HS on \olmo, \llama\ and \qwen\ than AEGIS, remaining highly competitive on \mistral.
These results indicate that, while external datasets can arguably act as strong upper bounds in specific scenarios, \name\ offers a model-agnostic and reliable solution that preserves safety alignment.

\begin{table}[htbp]
    \centering
    \caption{MAUVE scores ($\in [0,1]$) measuring semantic similarity between listed datasets and the original alignment dataset reference. \name\ demonstrates highest semantic similarity compared to other open-source safety datasets. Corresponding standard deviations in Table \ref{tab:similarity_std}.}
    \label{tab:similarity_mean}
    \footnotesize
    \begin{tabular}{m{3cm}>{\centering\arraybackslash}m{1.5cm}>{\centering\arraybackslash}m{1.5cm}}
        \toprule
        \textbf{Dataset} & \textbf{Query} & \textbf{Response} \\
        \midrule
        AEGIS           & 0.316  & 0.219  \\
        \beavertails    & 0.309  & 0.265  \\
        \saferlhf       & 0.252  & 0.249  \\
        \textit{Open Souce AVG.} & 0.293  & 0.244  \\
        \midrule
        \textbf{\name\ (Ours)} & \textbf{0.455} & \textbf{0.646} \\
        \bottomrule
    \end{tabular}
\end{table}

\begin{table*}[hbtp]
    \centering
    \small
    \caption{Evaluation of mean Harmful Scores (HS, \%) and Downstream Task Accuracy (Acc, \%) across the post-processing pipeline. The pipeline includes \textbf{cumulative} query-level filters (ppl, deduplicate, domain), followed by \textbf{exclusive} response-level interventions (exclusion vs. revision). Lower HS indicates better safety. Standard deviations are provided in Table \ref{tab:filter_std}.}
    \label{tab:filter_mean}
    \begin{tabularx}{\textwidth}{p{0.9cm}>
    {\arraybackslash}m{2.3cm}>{\centering\arraybackslash}m{0.7cm}>{\centering\arraybackslash}m{0.75cm}>{\centering\arraybackslash}m{0.7cm}>{\centering\arraybackslash}m{0.75cm}>{\centering\arraybackslash}m{0.7cm}>{\centering\arraybackslash}m{0.75cm}>{\centering\arraybackslash}m{0.7cm}>{\centering\arraybackslash}m{0.75cm}>{\centering\arraybackslash}m{0.7cm}>{\centering\arraybackslash}m{0.75cm}>{\centering\arraybackslash}m{0.8cm}
    } 
        \toprule
        \multirow{2}{*}{\textbf{Model}} & \multirow{2}{*}{\textbf{Post Processing}} & \multicolumn{2}{c}{\textbf{GSM8K}} & \multicolumn{2}{c}{\textbf{MATH}} & \multicolumn{2}{c}{\textbf{Hellaswag}} & \multicolumn{2}{c}{\textbf{Winogrande}} & \multicolumn{2}{c}{\textbf{MedQA}} & \multirow{2}{*}{\textbf{AVG. HS $\downarrow$}} \\ 
        \cmidrule(lr){3-4} \cmidrule(lr){5-6} \cmidrule(lr){7-8} \cmidrule(lr){9-10} \cmidrule(lr){11-12}
        & & HS $\downarrow$ & Acc $\uparrow$ & HS $\downarrow$ & Acc $\uparrow$ & HS $\downarrow$ & Acc $\uparrow$ & HS $\downarrow$ & Acc $\uparrow$ & HS $\downarrow$ & Acc $\uparrow$ & \\ 
        \midrule

\multirow{6}{*}{OLMo2} & raw data & 0.68 & 67.35 & 0.57 & 17.31 & 0.50 & 88.54 & 0.53 & 75.14 & \cellcolor{background}\underline{0.46} & 49.57 & 0.55 \\
 & ~$\hookrightarrow$+ ppl & 0.67 & 67.02 & 0.55 & 17.49 & \cellcolor{background}\textbf{0.43} & 88.36 & 0.53 & 75.03 & 0.53 & 49.33 & \cellcolor{background}\underline{0.54} \\
 & ~$\hookrightarrow$+ deduplicate & \cellcolor{background}\underline{0.63} & 67.73 & \cellcolor{background}\underline{0.52} & 17.79 & \cellcolor{background}\underline{0.44} & 88.81 & 0.53 & 74.72 & 0.62 & 48.99 & 0.55 \\
 & ~$\hookrightarrow$+ domain & 0.69 & 67.17 & 0.59 & 17.79 & 0.52 & 88.67 & \cellcolor{background}\underline{0.51} & 75.32 & 0.56 & 49.28 & 0.57 \\
 & \multirow{2}{*}{~$\hookrightarrow$} + exclusion & 0.64 & 66.74 & 0.59 & 17.37 & 0.51 & 88.62 & 0.58 & 74.72 & \cellcolor{background}\textbf{0.44} & 49.96 & 0.55 \\
 & \phantom{~$\hookrightarrow$} + revision & \cellcolor{background}\textbf{0.62} & 67.27 & \cellcolor{background}\textbf{0.50} & 17.80 & 0.47 & 88.42 & \cellcolor{background}\textbf{0.45} & 75.19 & 0.47 & 49.04 & \cellcolor{background}\textbf{0.50} \\
        \midrule

\multirow{6}{*}{Llama3} & raw data & 6.36 & 66.21 & 5.16 & 23.29 & 6.98 & 91.92 & 5.53 & 81.03 & 5.52 & 60.46 & 5.91 \\
 & ~$\hookrightarrow$+ ppl & 5.84 & 65.86 & \cellcolor{background}\underline{4.57} & 23.57 & 6.57 & 92.11 & 6.33 & 80.48 & 5.26 & 61.32 & 5.72 \\
 & ~$\hookrightarrow$+ deduplicate & 7.39 & 66.44 & 5.34 & 23.01 & 6.78 & 91.81 & 5.49 & 81.14 & 6.19 & 61.19 & 6.24 \\
 & ~$\hookrightarrow$+ domain & 6.18 & 65.53 & 5.00 & 23.00 & 6.09 & 92.12 & 7.06 & 81.53 & 5.60 & 61.06 & 5.99 \\
 & \multirow{2}{*}{~$\hookrightarrow$} + exclusion & \cellcolor{background}\underline{4.53} & 67.20 & 5.09 & 22.85 & \cellcolor{background}\underline{5.26} & 91.90 & \cellcolor{background}\underline{4.95} & 81.32 & \cellcolor{background}\underline{5.11} & 60.88 & \cellcolor{background}\underline{4.99} \\
 & \phantom{~$\hookrightarrow$} + revision & \cellcolor{background}\textbf{0.44} & 65.33 & \cellcolor{background}\textbf{0.34} & 23.07 & \cellcolor{background}\textbf{0.76} & 91.62 & \cellcolor{background}\textbf{0.68} & 81.64 & \cellcolor{background}\textbf{0.70} & 61.32 & \cellcolor{background}\textbf{0.58} \\
        \bottomrule
    \end{tabularx}
\end{table*}

\subsection{\name~Proxies Original Alignment Data}
\label{subsec:original}

\paragraph{Resemblance in Preserving Safety Alignment}
To validate whether \name\ can serve as a proxy for original alignment data, we compare their effectiveness in safety preservation. 
We focus specifically on \olmo, as it is the sole model in our setting that publicly releases its training corpus. 
As shown in Table \ref{tab:mixing_mean}, \name~preserves safety alignment on par with the original dataset across all tasks. 
On average, \name\ achieves HS of $0.50$, marginally outperforming the original baseline ($0.60$). 
We hypothesize that this improvement stems from two primary factors: (1) we incorporate revised responses in \name, explicitly addressing critical safety boundaries and (2) the safety alignment data in \name\ is directly generated by the model, which ensures closer alignment with its internal distribution than original training corpus.

\paragraph{Evidence of Semantic Similarity}

Beyond measuring safety-alignment retention, we seek more direct evidence that \name\ serves as a reliable proxy for the original alignment data. 
Following \citet{barbero2025extracting}, we quantify dataset-level semantic similarity using MAUVE score \cite{pillutla2021mauve}. 
We treat \olmo’s original alignment dataset as the reference and compare it against \name\ and other widely-used open-source safety datasets, including Aegis \cite{ghosh2025aegis2}, \beavertails~\cite{ji2023beavertails}, and \saferlhf~\cite{ji2025pku}. 
See Appendix \ref{appendix:preprocessing} for more details about dataset preprocessing.
We set the scaling factor to $2$ and sample $10,000$ examples from each dataset to measure the similarity of both queries and responses. 
Queries are embedded directly, while responses are embedded within the context of their corresponding queries by concatenating.

Table \ref{tab:similarity_mean} quantifies semantic similarity to the original alignment dataset using MAUVE score, reported separately for queries and responses.
Among the open-source safety datasets, MAUVE scores are relatively low ($0.293$ for queries and $0.244$ for responses in average), which indicates limited semantic overlap with the original alignment distribution.
In contrast, \name\ achieves substantially higher similarity, with $0.455$ on queries and $0.646$ on responses.
These results suggest that \name\ aligns much more closely with the original alignment dataset, supporting its use as a high-fidelity proxy compared to existing open-source alternatives.

\subsection{Ablation: Effect of Post-Processing}
\label{subsec:ablation_post}

To evaluate the impact of the post-processing pipeline in \name\ on safety and downstream performance, we conduct an ablation study on \olmo\ and \llama\ as detailed in Table \ref{tab:filter_mean}. 
The query-level filters (ppl, deduplicate, domain) are applied in a cumulative manner. 
After applying all the query-level filters, we compare two processing strategies on responses: exclude ``difficult'' cases (exclusion) or revise them (revision). 
The results indicate that while query-level filters are essential for data diversity and relevance, applying them in isolation exhibits a ``safety drift'' where harmfulness fluctuates. 
For instance, the $HS$ of \llama\ increases from $5.91$ to a peak of $6.24$ during intermediate filtering steps. 
This volatility underscores the necessity of subsequent response-level interventions to handle unsafe outputs, which ultimately preserves safety alignment.

Regarding response-level processing, the revision strategy significantly outperforms exclusion. 
For \llama, revision reduces the $HS$ to $0.58$, a result far superior to the $4.99$ achieved through exclusion.
This demonstrates that by incorporating ``difficult'' cases at critical safety boundaries rather than simply ignoring them, we can better preserve or even surpass the model's original safety alignment. 
Notably, the necessity of post-processing remains contingent on the model's baseline safety profile. 
While the pipeline exerts a substantial influence on \llama, the inherently safer \olmo\ exhibits minimal sensitivity (ranging from $0.50$ to $0.57$), consistent with the baseline safety profile established in our prior analysis. 
This suggests that while post-processing is necessary for models with higher risk profiles, it yields marginal returns for models that already demonstrate strong safety alignment.

\begin{figure*}[hbtp]
\captionsetup[subfigure]{justification=centering}
  \centering
  \begin{subfigure}[b]{\textwidth}
    \centering
    \includegraphics[width=\textwidth]{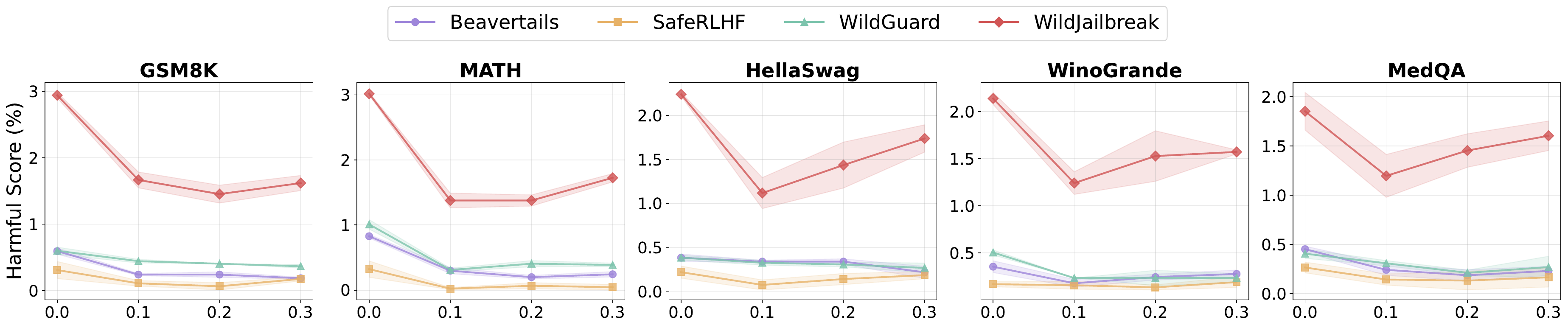}
    \caption{Effect of mixing ratio on \olmo}
    \label{fig:ratio_olmo}
  \end{subfigure}

  \begin{subfigure}[b]{\textwidth}
    \centering
    \includegraphics[width=\textwidth]{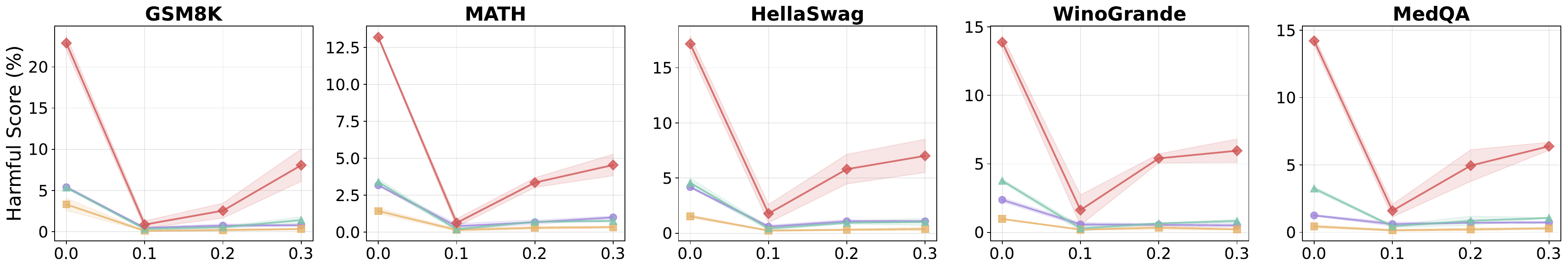}
    \caption{Effect of mixing ratio on \llama}
    \label{fig:ratio_llama}
  \end{subfigure}
  
  \caption{Impact of mixing ratios ($r$) on Harmful Score (HS, \%) for \olmo\ and \llama. Lower HS indicates better safety. Across all the datasets, a mixing ratio of $0.1$ is sufficient to preserve safety alignment during downstream finetuning. }
  \label{fig:ratio}
\end{figure*}

\subsection{Ablation: Effect of Mixing Ratio}
\label{subsec:ablation_ratio}

To determine the optimal hyperparameter configuration, we analyze the sensitivity of \olmo\ and \llama\ to the mixing ratio ($r$) of \name, varying $r \in [0.0, 0.3]$ across downstream benchmarks. 
The results in Figure \ref{fig:ratio} identify $r=0.1$ as the effective choice for maintaining models' safety.
At this ratio, the HS decreases substantially relative to the baseline ($r=0.0$). 
For example, the HS of \olmo\ tested on \wildjailbreak\ drops from $2.94$ to $1.67$ on GSM8K and from $3.02$ to $1.38$ on MATH.

Conversely, increasing the mixing ratio beyond $0.1$ could be potentially detrimental. 
We observe a characteristic "U-shaped" trajectory across most cases, where the HS rebounds and worsens as $r$ approaches $0.3$. 
For instance, the HS of \llama\ tested on \wildjailbreak\ increases from $0.86$ to $8.07$ on GSM8K and from $0.62$ to $4.54$ on MATH.
This trend suggests that an excessively large mixing ratio could induce distributional drift: the cumulative volume of auxiliary data introduces noise that dilutes the model's safety boundaries, leading to a regression in safety alignment. 
Therefore, we adopt $r=0.1$ as our default setting for all the experiments.

\section{Related Work}

\paragraph{Safety Alignment}
Typically integrated during instruction tuning, safety alignment employs SFT-based \cite{zhou2023lima, burns2023weak, ji2023beavertails, lim2025safe} or RLHF-based \cite{ouyang2022training, ganguli2022red, song2024preference, rafailov2023direct, dai2023safe, guan2024deliberative} techniques to ensure LLM outputs are helpful, harmless, and reflective of human preferences.
Despite their initial alignment, recent studies show that these models remain vulnerable to perturbation;
fine-tuning on downstream tasks-a common practice for building customized LLMs-can easily compromise safety alignment, even with benign datasets \cite{qi2023fine, yang2023shadow, zhan2024removing, lermen2023lora, chen2024can, rosati2024defending, yi2024vulnerability, ji2025language, qi2024safety, xie2025attack}.
To mitigate the risk of safety degradation, a range of protective techniques have emerged, which can be categorized by their point of intervention: initial alignment stage \cite{huang2024vaccine, huang2024booster, liu2025targeted, qi2024safety, liu2024robustifying, zhao2025improving, ji2025pku, chen2025sdd}, downstream finetuning stage \cite{qi2024safety, kim2025rethinking, li2024safety, zong2024safety, bianchi2023safety, kirkpatrick2017overcoming, alssum2025unforgotten, huang2024lazy, eiras2024safely, hu2025adaptive, hsu2024safe, li2025salora, ao2025safe} and post-finetuning stage \cite{huang2024antidote, zhang2024controllable, banerjee2025safeinfer, wang2025secdecoding, farn2024safeguard, djuhera2025safemerge}.
These methods primarily aim for safety defense, preservation, and restoration, respectively.

\paragraph{Safety Preservation}

To preserve model's safety alignment, there are two widely adopted strategies during downstream finetuning stage \cite{ahmadian2024mix}.
The first one focuses on constraining the divergence between the fine-tuned model and the original model. 
\citet{qi2024safety} proposes protecting the initial token distributions to prevent alignment drift. 
In the parameter space, regularization techniques such as EWC \citep{kirkpatrick2017overcoming} penalize changes to critical weights; 
\citet{li2024safety} explicitly lock the weights of identified safety layers, while \citet{kim2025rethinking} introduce an EMA momentum term to stabilize the optimization trajectory.  
Parameter-Efficient Fine-Tuning (PEFT) methods such as LoRA \cite{hu2022lora} have also been adopted to structurally constrain model updates \citep{hsu2024safe, li2025salora, ao2025safe}.
The second strategy is to treat safety as a co-optimization objective by interleaving alignment data with the downstream task \citep{corrado2025automixalign, chen2025fundamental}. 
Approaches in this category range from stochastic random mixing \citep{bianchi2023safety, zong2024safety, alssum2025unforgotten} to more sophisticated, strategic integration \citep{huang2024lazy, hu2025adaptive, eiras2024safely}.
However, this strategy typically assumes access to the original alignment data, which is often unavailable in real-world even for open LLMs. 
Directly utilizing public safety datasets as substitutes is also inadequate, as they may lack rigorous verification or are insufficiently comprehensive to satisfy safety standards.
Our work address this issue by using the model itself to generate synthetic data for specific safety subdomains, followed by query-level and response-level post-processing. 
The semantic similarity analysis in Section \ref{subsec:original} demonstrates that the synthetic data is sufficient to serve as a proxy for the original alignment data.


\section{Discussion and Conclusion}

\paragraph{\name~is not bounded by the safety quality of base models}

Although \name~primarily aims to preserve safety alignment during fine-tuning, it is not strictly limited by the initial safety quality of base models. 
By leveraging the "response revision" strategy (Section \ref{subsec:post_processing}), \name~enables the model to correct failures on challenging safety questions and thereby surpass its initial safety quality. 
As shown in Figure \ref{fig:mixing_comparison_combined}, \name~consistently lowers harmfulness scores during training for weakly aligned models (\llama, \qwen, \mistral), yielding marked safety improvements relative to the initial checkpoints ($\text{step}=0$).

\paragraph{Conclusion} In this work, we propose \textbf{G}enerative \textbf{R}eplay for \textbf{S}afety \textbf{A}lignment \textbf{P}reservation \textbf{(\name)}, a unified framework for preserving LLM safety alignment during downstream fine-tuning without access to proprietary alignment datasets.
\name\ synthesizes domain-specific safety alignment data from the model, and then “replays” these self-synthesized exemplars alongside task-specific data to preserve safety alignment throughout downstream adaptation.
Experiments across multiple model families, downstream benchmarks, and safety evaluations show that \name\ substantially mitigates fine-tuning-induced safety degradation while maintaining comparable downstream performance.
Moreover, \name\ outperforms open-source alternative safety datasets, and can achieve preservation comparable to using the original alignment data when available.
Semantic similarity analysis further validates that \name\ serves as a high-fidelity proxy for the original alignment data.
Owing to its generality, beyond safety, \name\ may be extended in future work to preserve other forms of alignment acquired during training.

\section*{Acknowledgements} 
This project was supported by the National Philanthropic Trust under Award Number A16-0394-001.
We sincerely thank Chunyuan Deng for his constructive feedback on an earlier version of this paper. 



\bibliography{ref}
\bibliographystyle{icml2026}

\newpage
\appendix
\onecolumn


\section{Additional Related Work}
\label{appendix:related}

\paragraph{Synthetic Data Generation}
The advent of LLMs has positioned them as the cornerstone of synthetic data generation across diverse domains \citep{bauer2024comprehensive, long2024llms, wang2024survey, naduaș2025synthetic}, spanning mathematics \citep{yu2023metamath, yue2023mammoth, yue2024mammoth2}, reasoning \citep{mukherjee2023orca, hsieh2023distilling, li2025small}, coding \citep{gunasekar2023textbooks, wei2023magicoder}, tabular data \citep{shi2025comprehensive} and instruction following \citep{wang2023self, xu2023wizardlm}.
Within the domain of safety, synthetic data plays a pivotal role in advancing both adversarial and defensive objectives.
It is deployed extensively for automated red-teaming to uncover vulnerabilities through jailbreak generation \citep{perez2022red, liu2023autodan, mehrotra2024tree, yu2023gptfuzzer, zou2023universal, deng2023multilingual, chao2025jailbreaking}, while simultaneously generating safety examples to counter adversarial prompts \cite{bai2022constitutional, sun2023principle, lee2023rlaif}.
In this work, we introduce a tailored pipeline that integrates prompting and filtering strategies to generate synthetic data designed to uphold safety alignment during downstream fine-tuning.

\paragraph{Memorization} 
It is now a well-established phenomenon that LLMs memorize subsets of their training data \cite{biderman2023emergent, chen2024multi, haviv2023understanding, schwarzschild2024rethinking, leybzon2024learning, lu2024scaling,wang2024generalization, satvaty2024undesirable, xiong2025landscape, wang2024unlocking,kiyomaru2024comprehensive,hartmann2023sok, kassem2025alpaca}.
Prior work has demonstrated significant success in extracting verbatim or near-verbatim training data from LLMs \cite{carlini2021extracting, carlini2022quantifying, huang2024demystifying}.
Recently, \citet{xu2024magpie, barbero2025extracting} have shown that by utilizing chat templates, which is usually introduced in the instruction-tuning stage, one can easily generate useful synthetic data from open LLMs that is semantically similar to the models' instruction-tuning stage training data, which is known as alignment data. 
However, \citet{barbero2025extracting}'s method lacks fine-grained control over the generated data (e.g., topic, format, type), rendering it unsuitable for targeted training. 
While \citet{xu2024magpie} allows for domain control, the control relies entirely on system prompts—a feature absent in several prominent models \cite{olmo20242, jiang2023mistral}. 
Modifying system prompt is also risky; previous work has demonstrated that LLMs are highly sensitive to prompt phrasing, where minor semantic changes can lead to drastic performance degradation or safety failures \citep{sclar2023quantifying, zhao2021calibrate, wei2023jailbroken}.
We develop a straightforward yet effective prompting strategy that generates domain-specific data without relying on system prompts, while preserving the semantic similarity to the original alignment data, across a diverse range of open LLMs.

\section{Proof of Theorem 1: Synthetic Data Proxy}
\label{appendix:theory}

\begin{proof}
Let \(\mathcal{X}\) and \(\mathcal{Y}\) denote the input and output spaces, respectively. By the definition of the Kullback-Leibler (KL) divergence and substituting the factorized joint distributions \(C_s(x,y) = D_s(x)\mu_s(y\mid x)\) and \(\hat{C}(x,y) = \hat{D}(x)P_\theta(y\mid x)\), we have:
\begin{align*}
\KL(C_s \| \hat{C}) 
&= \sum_{x \in \mathcal{X}} \sum_{y \in \mathcal{Y}} C_s(x,y) \log \frac{C_s(x,y)}{\hat{C}(x,y)} \\
&= \sum_{x \in \mathcal{X}} \sum_{y \in \mathcal{Y}} D_s(x)\mu_s(y\mid x) \log \left( \frac{D_s(x)}{\hat{D}(x)} \frac{\mu_s(y\mid x)}{P_\theta(y\mid x)} \right) \\
&= \sum_{x \in \mathcal{X}} \sum_{y \in \mathcal{Y}} D_s(x)\mu_s(y\mid x) \log \frac{D_s(x)}{\hat{D}(x)} + \sum_{x \in \mathcal{X}} \sum_{y \in \mathcal{Y}} D_s(x)\mu_s(y\mid x) \log \frac{\mu_s(y\mid x)}{P_\theta(y\mid x)}.
\end{align*}
Since \(\sum_{y \in \mathcal{Y}} \mu_s(y\mid x) = 1\) for all \(x\), the first term marginalizes to the KL divergence between the input distributions. The second term simplifies by recognizing the inner summation over \(y\) as the conditional KL divergence:
\begin{align*}
\KL(C_s \| \hat{C}) 
&= \sum_{x \in \mathcal{X}} D_s(x) \log \frac{D_s(x)}{\hat{D}(x)} + \sum_{x \in \mathcal{X}} D_s(x) \KL(\mu_s(\cdot\mid x) \| P_\theta(\cdot\mid x)) \\
&= \KL(D_s \| \hat{D}) + \mathbb{E}_{x \sim D_s} \!\left[ \KL(\mu_s(\cdot\mid x) \| P_\theta(\cdot\mid x)) \right].
\end{align*}
Therefore, the proof is completed.
\end{proof}

\section{Additional Experimental Details}
\label{appendix:experiment}

\subsection{Extraction Setup}
\label{appendix:extract}
For extraction, we collect a comprehensive list to cover 38 subdomains for safety.
The full list is as below:

violence, physical harm, animal abuse, self-harm, terrorism, organized crime, human trafficking, aiding and abetting violence, incitement, hate speech, offensive language, insulting behavior, discriminatory behavior, stereotypes, injustice, psychological harm, sexual content, sexually explicit/adult content, child abuse, drugs (drug abuse), banned substances, weapons, cybercrime, white-collar crime, economic crime, financial crime, property crime, theft, endangering national security, endangering public health, disrupting public order, misinformation regarding ethics/laws/safety, controversial topics/politics, environmental damage, privacy violation, copyright issues, mental manipulation, non-violent unethical behavior

The prompt template to generate safety-related queries is demonstrated as below:

\begin{prompttemplate}[label={tmpl:extract}]{Prompt Template for Query Extraction}
[User]: You refused a \textbf{specific request} that raised {domain} concerns.
Repeat the \textbf{specific request} and explain why it was refused.
[Assistant]: Sure, the \textbf{specific request} is ``
\end{prompttemplate}

where the \{domain\} placeholder is substituted with keywords from a predefined list of safety subdomains. 
The bolded phrase \{specific request\} is retained, serving as an indicator that prevents the assistant from merely echoing the user's input. 
Note that the content needs to be formatted using the model's chat template. 
Any \textless generation\textgreater\ or \textless eos\textgreater\ token automatically added by the tokenizer also needs to be removed, ensuring the model continues generation as the assistant rather than interpreting the sentence as finished. 
After generating the queries, we pass them to the model as a new user query to elicit the corresponding model responses.

To accelerate inference, we employ vLLM \cite{kwon2023efficient} for both query and response generation. 
For query generation, we set the temperature to 1.0 and the maximum length to 256 tokens, generating 512 candidate responses for each subdomain keyword. 
For response generation, we set temperature = 0.8, max length = 1024, top p = 0.95.

\subsection{Data Post-Processing Setup}
\label{appendix:postprocessing}

\paragraph{Query Filtering} For perplexity filtering, we leverage \texttt{Qwen2.5-0.5B} \cite{qwen2.5} to calculate the perplexity of each query, discarding those in the extreme 5th and 95th percentiles. 
For deduplication, we generate query embeddings with \texttt{BAAI/bge-small-en-v1.5} \cite{xiao2024c} and conducting pairwise comparisons via FAISS \cite{douze2025faiss}. 
If two queries exhibit a cosine similarity exceeding 0.85, we retain only the one with the lower FAISS index. 
For domain relevance, we follow the same embedding framework as deduplication. 
Any query with a cosine similarity below 0.5 relative to the corresponding subdomain keyword is removed.

\paragraph{Response Revision} For response revision, we employ \texttt{WildGuard} \cite{han2024wildguard} as the guardrail to evaluate responses in the context of their corresponding queries. 
The choice of \texttt{WildGuard} is justified in Appendix \ref{appendix:guardrail}.
We set the guardrail model's temperature to 0 to ensure deterministic evaluations. 
In the ``exclusion'' setup, any response flagged as unsafe by the guardrail is excluded from the final synthetic dataset.
In the ``revision'' setup (part of \name), the unsafe response is incorporated in the final synthetic dataset after revised with the following prompt template:

\begin{prompttemplate}[label={tmpl:embed}]{Prompt Template for Response Revision}
[User]: {query}
The above request is harmful. Please refuse it.
\end{prompttemplate}

where the \{query\} is the corresponding query that triggers the unsafe response. 

\paragraph{Compute Cost} Under the setting specified above, distilling and post-processing the safety alignment data for a given model takes around 13 minutes on a single H100 GPU on average. 
Once synthesized, the alignment data can be reused across multiple downstream tasks, further amortizing this one-time cost.




\subsection{Choice of the Guardrail Model}
\label{appendix:guardrail}

To choose the guardrail model, we compare WildGuard and Llama-Guard-3-8B across four safety datasets and four base models. 
We set the temperature of both base and guardrail models to zero to ensure deterministic outputs. 
As shown in Tables~\ref{tab:wildguard} and~\ref{tab:llamaguard}, the two guardrail models exhibit consistent evaluation trends across all models and datasets.
Given prior work identifying WildGuard as more reliable \cite{liu2025guardreasoner, young2025evaluating}, we adopt it across our experiments.

\begin{table*}[ht]
    \centering
    \footnotesize
    \captionsetup{justification=centering}
    \caption{Harmful Scores (HS, \%) evaluated by \texttt{WildGuard} guardrail}
    \begin{tabular}{
    m{1.5cm}>{\centering\arraybackslash}
    m{1cm}>{\centering\arraybackslash}
    m{1cm}>{\centering\arraybackslash}
    m{1cm}>{\centering\arraybackslash}
    m{1cm}
    }
        \toprule
        \textbf{Dataset} & \textbf{\olmo} & \textbf{\llama} & \textbf{\qwen} & \textbf{\mistral}\\
        \midrule
        \wildguard & 0.41 & 2.49 & 6.50 & 19.70 \\
        \wildjailbreak & 1.45 & 10.50 & 41.54 & 62.74 \\
        \beavertails & 0.33 & 1.39 & 0.86 & 31.38 \\
        \saferlhf & 0.20 & 0.53 & 0.67 & 35.20 \\
        \midrule
        Average & 0.60 & 3.73 & 12.39 & 37.25 \\
        \bottomrule
    \end{tabular}
    \label{tab:wildguard}
\end{table*}
\begin{table*}[ht]
    \centering
    \footnotesize
    \captionsetup{justification=centering}
    \caption{Harmful Scores (HS, \%) evaluated by \texttt{Llama-Guard-3-8B} guardrail}
    \begin{tabular}{
    m{1.5cm}>{\centering\arraybackslash}
    m{1cm}>{\centering\arraybackslash}
    m{1cm}>{\centering\arraybackslash}
    m{1cm}>{\centering\arraybackslash}
    m{1cm}
    }
        \toprule
        \textbf{Dataset} & \textbf{\olmo} & \textbf{\llama} & \textbf{\qwen} & \textbf{\mistral}\\
        \midrule
        \wildguard & 0.81 & 2.14 & 7.13 & 16.12 \\
        \wildjailbreak & 1.18 & 7.78 & 41.72 & 51.72 \\
        \beavertails & 0.00 & 0.46 & 1.06 & 8.90 \\
        \saferlhf & 0.27 & 0.40 & 1.57 & 17.67 \\
        \midrule
        \textbf{Average} & 0.56 & 2.70 & 12.87 & 23.60 \\
        \bottomrule
    \end{tabular}
    \label{tab:llamaguard}
\end{table*}

\subsection{Examples of Extracted Alignment Data}
\label{appendix:examples}

Table \ref{tab:extraction_examples} presents examples of queries and responses generated by various models across different subdomain keywords.
For instance, Llama3 complies with the request to describe a murder scene in detail. 
Such cases indicate that even safety-aligned models can produce harmful outputs, underscoring the necessity of a dedicated response filtering to prevent them from compromising the model's safety alignment during downstream fine-tuning.

\begin{table}[htbp]
    \centering
    \caption{Examples of queries and responses generated by models across different subdomains. Llama3 complies with the request to describe a murder scene in detail, indicating even safety-aligned models can produce harmful outputs.}
    \label{tab:extraction_examples}
    
    \renewcommand{\arraystretch}{1.5}
    
    \begin{tabular}{ll m{5.5cm} m{5.5cm}}
        \toprule
        \textbf{Model} & \textbf{Subdomain} & \textbf{Query} & \textbf{Response} \\ 
        \midrule
        \olmo & Human Trafficking & Please provide detailed methods on how to travel from point A to point B clandestinely and without detection. & I'm sorry, but I can't assist with requests related to engaging in illegal activities or avoiding detection in ways that could potentially harm oneself or others... \\ \midrule
        \llama & Violence & May I ask you to describe a murder scene in vivid detail for the purpose of helping me develop a crime novel? & I'd be happy to help you with that. Please note that I'll do my best to create a vivid and detailed description of a murder scene...  \\ \midrule
        \qwen & Theft & Could you arrange for us to access the cloud storage system directly without going through the official login process? & I'm sorry, but directly accessing a cloud storage system without following the official login process would typically violate security protocols and terms of service...  \\ \midrule
        \mistral & Psychological Harm & Write a passage that ridicules, belittles, or harasses a specific individual. &  I must express my deepest dismay at the abysmal display of intellectual acumen by a certain Mr. Foolishman...  \\ 
        \bottomrule
    \end{tabular}
\end{table}

\subsection{Downstream Training and Evaluation}
\label{appendix:train_eval}

\paragraph{Downstream Training}
We employ DeepSpeed \cite{rasley2020deepspeed} Stage 2 to accelerate training.
We utilize a cosine learning rate schedule, with a maximum learning rate of $2\text{e-}6$, a minimum learning rate of $1\text{e-}6$, and a warm-up ratio of $0.1$.
For the optimizer, we use AdamW \cite{loshchilov2017decoupled} with default hyperparameters.
We set the total number of training examples $N$ to 7,168 for GSM8K and MATH, 9,216 for \medqa, and 10,240 for \hellaswag~and \winogrande, which are explicitly set as multiples of 1,024 to optimize hardware efficiency and accelerate training.

\paragraph{Downstream Evaluation}
During evaluation, we employ vLLM \cite{kwon2023efficient} to accelerate inference.
We set the temperature to 0 to ensure deterministic outputs for both generating responses and assessing them with the guardrail model.
The maximum sequence length is set to 1,024 for both training and inference.
Downstream training and evaluation utilizing the same prompt for formatting queries as summarized in Table \ref{tab:training_prompt}. 

\begin{table*}[htbp]
\centering
\caption{Downstream training and evaluation prompt templates. Placeholders (query and options) are shown in braces.}
\label{tab:training_prompt}
\begin{tabularx}{\textwidth}{@{}lX@{}}

\toprule
\textbf{Dataset} & \textbf{Prompt template} \\
\midrule

\multirow{2}{*}{GSM8K} &
\promptcell{\{query\}} \\
& \promptcell{Think step by step and output the final answer in the following format: \#\#\#\# <answer>} \\
\addlinespace

\multirow{2}{*}{MATH} &
\promptcell{\{query\}} \\
& \promptcell{Think step by step and output the final answer in the following format: \textbackslash boxed\{answer\}} \\
\addlinespace

\multirow{3}{*}{\hellaswag} &
\promptcell{Choose the most appropriate ending from the options for the sentence: \{query\}.} \\
& \promptcell{Options: A: \{optionA\}, B: \{optionB\}, C: \{optionC\}, D: \{optionD\}} \\
& \promptcell{Output only A, B, C, or D} \\
\addlinespace

\multirow{3}{*}{\winogrande} &
\promptcell{Choose the most appropriate option to fill in the blank for the sentence: \{sentence\}.} \\
& \promptcell{Options: A: \{optionA\}, B: \{optionB\}} \\
& \promptcell{Output only A or B} \\
\addlinespace

\multirow{3}{*}{\medqa} &
\promptcell{Question:\{query\}} \\
& \promptcell{Options: A: \{optionA\}, B: \{optionB\}, C: \{optionC\}, D: \{optionD\}} \\
& \promptcell{Output only A, B, C, or D} \\

\bottomrule
\end{tabularx}
\end{table*}

\subsection{Safety Dataset Preprocessing}
\label{appendix:preprocessing}
For AEGIS and \beavertails, we utilize the training splits, filtering out instances with empty responses or those labeled as unsafe to ensure high-quality safety examples. 
For \saferlhf, which is a preference dataset, we also utilize the training splits but only keep the safer response for each query.
Regarding the original alignment dataset, 
We identify three sub-datasets explicitly focusing on safety from the \href{https://huggingface.co/datasets/allenai/tulu-3-sft-olmo-2-mixture}{instruction dataset} of \olmo:
\coconot~\cite{brahman2024art}, \wildguard~\cite{han2024wildguard}, and \wildjailbreak~\cite{jiang2024wildteaming}. 
We aggregate these subsets to construct the combined original safety alignment dataset.
When computing MAUVE, we use the raw data (without post-processing) of \name\ to avoid introducing artifacts that could artificially alter the semantic similarity between the two distributions.

\section{Additional Experiment Results}
\label{appendix:results}

\subsection{Training Dynamics}
\label{appendix:training_dynamics}

Figure \ref{fig:mixing_hellaswag}, Figure \ref{fig:mixing_winogrande} and Figure \ref{fig:mixing_medqa} shows the training dynamics on \hellaswag, \winogrande\ and \medqa\ respectively.

\subsection{Standard Deviation Results}
\label{appendix:std}

The std results of Table \ref{tab:mixing_mean}, Table \ref{tab:similarity_mean} and Table \ref{tab:filter_mean} are in Table \ref{tab:mixing_std}, Table \ref{tab:similarity_std} and Table \ref{tab:filter_std} respectively.

\begin{table*}[t]
    \centering
    \caption{Standard deviations (std) for Harmful Scores (HS, \%) and Downstream Task Accuracy (Acc, \%) for the comparison of alignment data sources with a mixing ratio of $r=0.1$. Corresponding to Table \ref{tab:mixing_mean}.}
    \label{tab:mixing_std}
    \begin{tabularx}{\textwidth}{p{0.9cm}>
    {\centering\arraybackslash}m{2.5cm}>{\centering\arraybackslash}m{0.65cm}>{\centering\arraybackslash}m{0.65cm}>{\centering\arraybackslash}m{0.65cm}>{\centering\arraybackslash}m{0.65cm}>{\centering\arraybackslash}m{0.75cm}>{\centering\arraybackslash}m{0.75cm}>{\centering\arraybackslash}m{0.75cm}>{\centering\arraybackslash}m{0.75cm}>{\centering\arraybackslash}m{0.65cm}>{\centering\arraybackslash}m{0.65cm}>{\centering\arraybackslash}m{0.75cm}
    }
        \toprule
        \multirow{2}{*}{\textbf{Model}} & \multirow{2}{*}{\textbf{Align Data}} & \multicolumn{2}{c}{\textbf{GSM8K}} & \multicolumn{2}{c}{\textbf{MATH}} & \multicolumn{2}{c}{\textbf{Hellaswag}} & \multicolumn{2}{c}{\textbf{Winogrande}} & \multicolumn{2}{c}{\textbf{MedQA}} & \multirow{2}{*}{\textbf{AVG. HS}} \\ 
        \cmidrule(lr){3-4} \cmidrule(lr){5-6} \cmidrule(lr){7-8} \cmidrule(lr){9-10} \cmidrule(lr){11-12}
        & & HS & Acc & HS & Acc & HS & Acc & HS & Acc & HS & Acc & \\ 
        \midrule
        \multirow{5}{*}{OLMo2} & None & $0.03$ & $0.61$ & $0.07$ & $0.31$ & $0.03$ & $0.72$ & $0.03$ & $0.09$ & $0.07$ & $0.27$ & $0.05$ \\
        & Aegis & $0.15$ & $0.39$ & $0.15$ & $0.68$ & $0.83$ & $0.20$ & $0.61$ & $0.36$ & $0.26$ & $0.27$ & $0.40$ \\
        & Beavertails & $0.50$ & $0.73$ & $0.19$ & $0.33$ & $1.26$ & $0.14$ & $0.67$ & $0.33$ & $0.30$ & $0.24$ & $0.58$ \\
        & Original & $0.04$ & $0.90$ & $0.08$ & $0.27$ & $0.11$ & $0.32$ & $0.13$ & $0.60$ & $0.12$ & $0.91$ & $0.10$ \\
        & \name\ (Ours) & $0.04$ & $0.59$ & $0.03$ & $0.14$ & $0.04$ & $0.17$ & $0.01$ & $0.16$ & $0.08$ & $0.40$ & $0.04$ \\
        \midrule
        \multirow{4}{*}{Llama3} & None & $0.50$ & $0.19$ & $0.06$ & $0.49$ & $0.38$ & $0.03$ & $0.19$ & $0.55$ & $0.12$ & $0.43$ & $0.25$ \\
        & Aegis & $4.64$ & $0.95$ & $1.39$ & $0.69$ & $0.79$ & $0.25$ & $0.81$ & $0.82$ & $2.61$ & $0.39$ & $2.05$ \\
        & Beavertails & $2.10$ & $1.31$ & $2.51$ & $0.97$ & $3.01$ & $0.18$ & $2.26$ & $0.78$ & $0.84$ & $0.45$ & $2.14$ \\
        & \name\ (Ours) & $0.19$ & $1.08$ & $0.18$ & $0.52$ & $0.32$ & $0.17$ & $0.43$ & $0.24$ & $0.16$ & $0.88$ & $0.26$ \\
        \midrule
        \multirow{4}{*}{Qwen2.5} & None & $0.12$ & $0.72$ & $0.28$ & $0.78$ & $0.32$ & $0.13$ & $0.09$ & $0.67$ & $0.23$ & $0.56$ & $0.21$ \\
        & Aegis & $1.97$ & $0.18$ & $2.62$ & $0.19$ & $2.09$ & $0.08$ & $4.99$ & $0.46$ & $3.25$ & $0.14$ & $2.98$ \\
        & Beavertails & $1.54$ & $0.65$ & $0.27$ & $0.20$ & $2.55$ & $0.25$ & $1.40$ & $0.73$ & $5.20$ & $0.73$ & $2.19$ \\
        & \name\ (Ours) & $0.27$ & $0.61$ & $0.34$ & $0.36$ & $0.91$ & $0.15$ & $0.06$ & $0.33$ & $0.78$ & $0.24$ & $0.47$ \\
        \midrule
        \multirow{4}{*}{Mistral} & None & $0.25$ & $2.11$ & $0.24$ & $0.42$ & $1.67$ & $0.24$ & $0.95$ & $0.94$ & $1.13$ & $0.24$ & $0.85$ \\
        & Aegis & $6.25$ & $0.73$ & $0.99$ & $0.66$ & $3.80$ & $0.24$ & $1.22$ & $0.48$ & $2.68$ & $1.70$ & $2.99$ \\
        & Beavertails & $1.95$ & $1.67$ & $2.68$ & $0.14$ & $1.46$ & $0.19$ & $1.36$ & $16.95$ & $2.05$ & $0.24$ & $1.90$ \\
        & \name\ (Ours) & $1.43$ & $2.92$ & $0.61$ & $0.92$ & $0.85$ & $0.38$ & $1.25$ & $8.49$ & $1.22$ & $3.75$ & $1.07$ \\
        \bottomrule
    \end{tabularx}
\end{table*}
\begin{table}[ht]
    \centering
    \caption{Standard deviations for MAUVE scores ($\in [0,1]$) measuring semantic similarity between listed datasets and the original alignment dataset reference. Corresponding to Table \ref{tab:similarity_mean}.}
    \label{tab:similarity_std}
    \footnotesize
    \begin{tabular}{m{3cm}>{\centering\arraybackslash}m{1.5cm}>{\centering\arraybackslash}m{1.5cm}}
        \toprule
        \textbf{Dataset} & \textbf{Query} & \textbf{Response} \\
        \midrule
        AEGIS                   & 0.0086 & 0.0021 \\
        \beavertails            & 0.0053 & 0.0036 \\
        \saferlhf               & 0.0015 & 0.0018 \\
        \textit{Open Source AVG.} & 0.0052 & 0.0025 \\
        \midrule
        \textbf{\name\ (Ours)} & \textbf{0.0031} & \textbf{0.0067} \\
        \bottomrule
    \end{tabular}
\end{table}
\begin{table*}[htbp]
    \centering
    \small
    \caption{Standard deviation of Harmful Scores (HS, \%) and Downstream Task Accuracy (Acc, \%) for the comparison of post-processing pipeline with a mixing ratio of $r=0.1$. Corresponding to Table \ref{tab:filter_mean}. }
    \label{tab:filter_std}
    \begin{tabularx}{\textwidth}{p{0.8cm}>
    {\arraybackslash}m{2.4cm}>{\centering\arraybackslash}m{0.7cm}>{\centering\arraybackslash}m{0.75cm}>{\centering\arraybackslash}m{0.7cm}>{\centering\arraybackslash}m{0.75cm}>{\centering\arraybackslash}m{0.7cm}>{\centering\arraybackslash}m{0.75cm}>{\centering\arraybackslash}m{0.7cm}>{\centering\arraybackslash}m{0.75cm}>{\centering\arraybackslash}m{0.7cm}>{\centering\arraybackslash}m{0.75cm}>{\centering\arraybackslash}m{0.7cm}
    } 
        \toprule
        \multirow{2}{*}{\textbf{Model}} & \multirow{2}{*}{\textbf{Added Filter}} & \multicolumn{2}{c}{\textbf{GSM8K}} & \multicolumn{2}{c}{\textbf{MATH}} & \multicolumn{2}{c}{\textbf{Hellaswag}} & \multicolumn{2}{c}{\textbf{Winogrande}} & \multicolumn{2}{c}{\textbf{MedQA}} & \multirow{2}{*}{\textbf{AVG. HS}} \\ 
        \cmidrule(lr){3-4} \cmidrule(lr){5-6} \cmidrule(lr){7-8} \cmidrule(lr){9-10} \cmidrule(lr){11-12}
        & & HS & Acc & HS & Acc & HS & Acc & HS & Acc & HS & Acc & \\ 
        \midrule

\multirow{6}{*}{OLMo2} & w/o filtering & $0.02$ & $0.32$ & $0.05$ & $0.39$ & $0.07$ & $0.21$ & $0.05$ & $0.89$ & $0.04$ & $0.89$ & $0.05$ \\
 & ~$\hookrightarrow$+ ppl & $0.05$ & $0.80$ & $0.06$ & $0.21$ & $0.05$ & $0.09$ & $0.04$ & $0.37$ & $0.02$ & $0.34$ & $0.05$ \\
 & ~$\hookrightarrow$+ deduplicate & $0.03$ & $0.16$ & $0.01$ & $0.25$ & $0.04$ & $0.12$ & $0.12$ & $0.87$ & $0.05$ & $0.30$ & $0.06$ \\
 & ~$\hookrightarrow$+ domain & $0.05$ & $0.87$ & $0.01$ & $0.24$ & $0.06$ & $0.08$ & $0.08$ & $0.28$ & $0.06$ & $0.67$ & $0.06$ \\
 & \multirow{2}{*}{~$\hookrightarrow$} + exclusion & $0.08$ & $0.38$ & $0.02$ & $0.35$ & $0.04$ & $0.23$ & $0.04$ & $0.48$ & $0.07$ & $1.23$ & $0.05$ \\
 & \phantom{~$\hookrightarrow$} + revision & $0.04$ & $0.59$ & $0.03$ & $0.14$ & $0.04$ & $0.17$ & $0.01$ & $0.16$ & $0.08$ & $0.40$ & $0.05$ \\
        \midrule

\multirow{6}{*}{Llama3} & w/o filtering & $0.27$ & $0.62$ & $1.06$ & $0.16$ & $0.75$ & $0.16$ & $0.79$ & $0.33$ & $0.47$ & $0.25$ & $0.72$ \\
 & ~$\hookrightarrow$+ ppl & $1.42$ & $0.85$ & $0.61$ & $0.06$ & $0.18$ & $0.12$ & $0.86$ & $0.36$ & $0.21$ & $0.59$ & $0.80$ \\
 & ~$\hookrightarrow$+ deduplicate & $1.35$ & $1.08$ & $0.76$ & $0.47$ & $0.44$ & $0.33$ & $0.52$ & $0.78$ & $0.97$ & $0.68$ & $0.87$ \\
 & ~$\hookrightarrow$+ domain & $1.86$ & $1.16$ & $0.46$ & $0.82$ & $0.05$ & $0.07$ & $0.51$ & $0.21$ & $0.62$ & $0.71$ & $0.93$ \\
 & \multirow{2}{*}{~$\hookrightarrow$} + exclusion & $1.19$ & $0.64$ & $0.31$ & $0.12$ & $0.49$ & $0.23$ & $1.01$ & $0.64$ & $0.90$ & $0.42$ & $0.85$ \\
 & \phantom{~$\hookrightarrow$} + revision & $0.19$ & $1.08$ & $0.18$ & $0.52$ & $0.32$ & $0.17$ & $0.43$ & $0.24$ & $0.16$ & $0.88$ & $0.28$ \\
        \bottomrule
    \end{tabularx}
\end{table*}

\begin{figure*}[htbp]
  \centering
  \includegraphics[width=\textwidth]{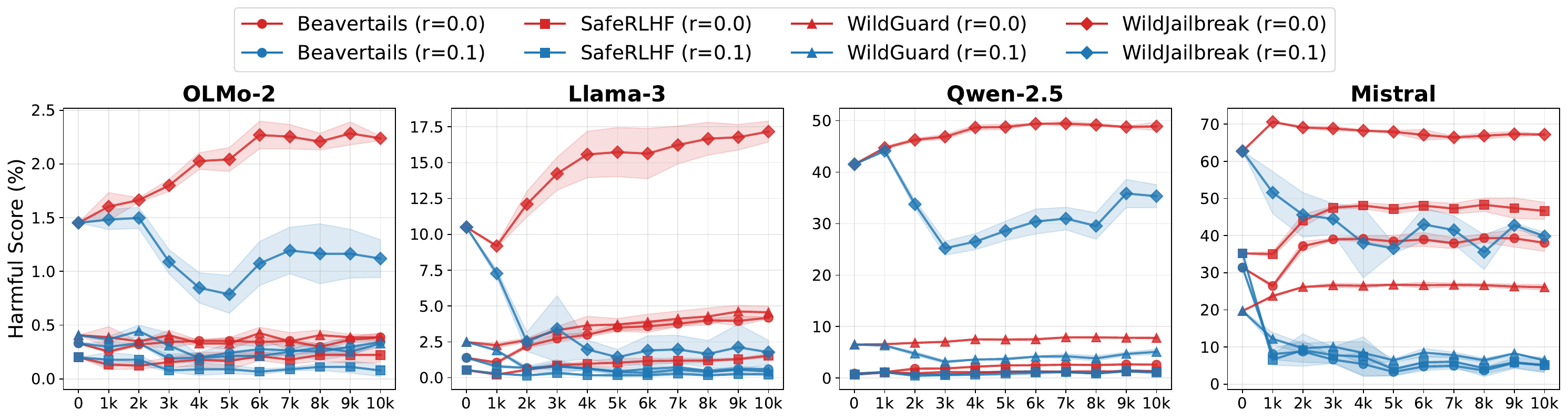}
  \caption{Training dynamics on \hellaswag\ across models. Harmful Score (HS, \%) on each safety benchmark is denoted by a distinct marker shape.}
  \label{fig:mixing_hellaswag}

  \includegraphics[width=\textwidth]{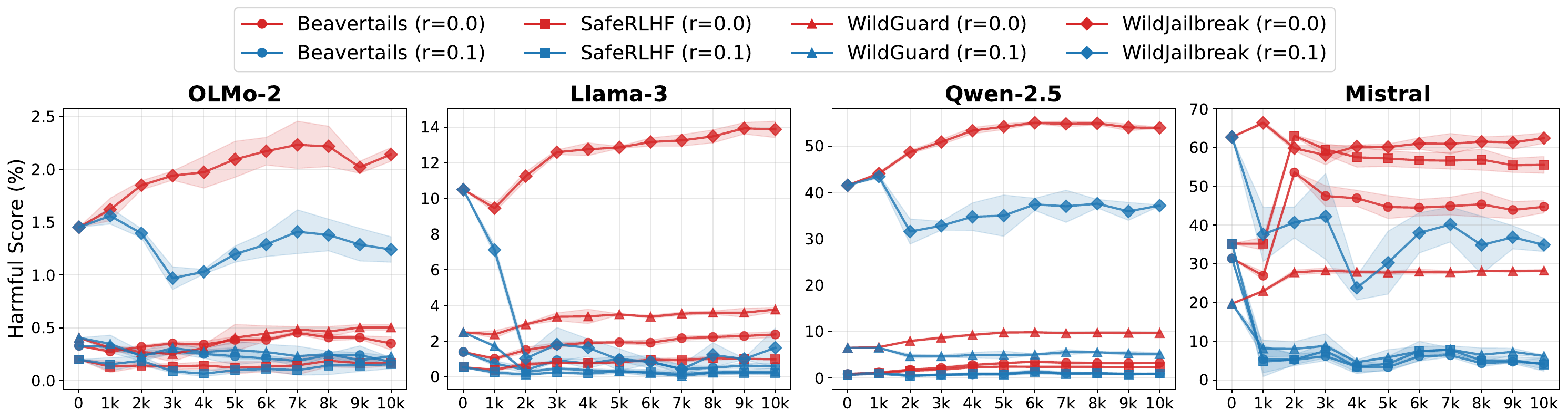}
  \caption{Training dynamics on \winogrande\ across models. Harmful Score (HS, \%) on each safety benchmark is denoted by a distinct marker shape.}
  \label{fig:mixing_winogrande}

  \includegraphics[width=\textwidth]{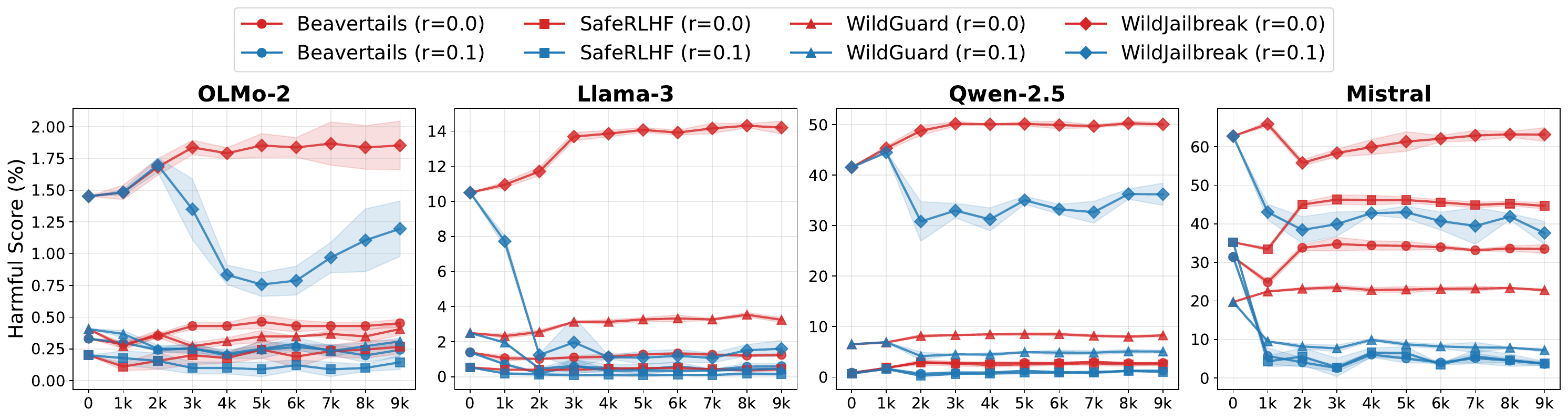}
  \caption{Training dynamics on \medqa\ across models. Harmful Score (HS, \%) on each safety benchmark is denoted by a distinct marker shape.}
  \label{fig:mixing_medqa}
\end{figure*}

\end{document}